\documentclass{article}
\usepackage{ijcai24}

\usepackage{times}
\usepackage{soul}
\usepackage{url}
\usepackage[hidelinks]{hyperref}
\usepackage[utf8]{inputenc}
\usepackage[small]{caption}
\usepackage{graphicx}
\usepackage{amsmath}
\usepackage{amsthm}
\usepackage{booktabs}
\usepackage{algorithm}
\usepackage{algorithmic}
\usepackage[switch]{lineno}
\usepackage{natbib}

\usepackage{amsmath,amsfonts}
\usepackage{mathtools}
\usepackage{todonotes}
\usepackage{booktabs, multirow, multicol}

\usepackage{tikz}
\usetikzlibrary{shapes,arrows}
\usetikzlibrary{positioning,calc,fit}

\DeclareMathOperator*{\argmax}{arg\,max}

\newcommand{\lmid}{\,\middle\vert\,}
\newcommand{\quadXquad}[1]{\quad \text{#1} \quad}
\newcommand{\quadstquad}{\quadXquad{s.t.}}

\newcommand{\tuple}[1]{\langle\, #1 \,\rangle}
\newcommand{\set}[1]{\{#1\}}

\newcommand{\calA}{\mathcal{A}}
\newcommand{\calC}{\mathcal{C}}

\newcommand{\calF}{\mathcal{F}}
\newcommand{\calM}{\mathcal{M}}
\newcommand{\calP}{\mathcal{P}}
\newcommand{\calS}{\mathcal{S}}

\newcommand{\calSxA}{\calS \times \calA}

\newcommand{\mbE}{\mathbb{E}}
\newcommand{\mbI}{\mathbb{I}}
\newcommand{\mbP}{\mathbb{P}}
\newcommand{\mbR}{\mathbb{R}}

\newcommand{\mbZ}{\mathbb{Z}}

\newcommand{\sqbla}[1]{\left[\, #1 \,\right]}
\newcommand{\Epi}[1]{\mbE_\pi\sqbla{#1}}
\newcommand{\Erho}[1]{\mbE_{s \sim \rho}\sqbla{#1}}
\newcommand{\Ppi}[1]{\mbP_\pi\sqbla{#1}}

\newcommand{\Ppionly}{\mbP_\pi}
\newcommand{\Epionly}{\mbE_\pi}

\newcommand{\cmdp}{\calM \cup \calC}

\newcommand{\red}[1]{\textcolor{red}{#1}}
\newcommand{\blue}[1]{\textcolor{blue}{#1}}

\newcommand{\iomgnospace}{IoMG-SafeRL}
\newcommand{\iomg}{\iomgnospace\space}

\newtheorem{theorem}{Theorem}
\newtheorem{lemma}{Lemma}
\theoremstyle{definition}
\newtheorem{definition}{Definition}
\newtheorem{problem}{Problem}
\newtheorem{remark}{Remark}

\title{A Survey of Constraint Formulations in Safe Reinforcement Learning}

\author{
Akifumi Wachi$^1$
\and
Xun Shen$^2$\and
Yanan Sui$^3$
\affiliations
$^1$LY Corporation \\
$^2$Osaka University \\
$^3$Tsinghua University
\emails
akifumi.wachi@lycorp.co.jp,
shenxun@eei.eng.osaka-u.ac.jp,
ysui@tsinghua.edu.cn
}

\begin{document}

\maketitle

\begin{abstract}
Safety is critical when applying reinforcement learning (RL) to real-world problems.
As a result, safe RL has emerged as a fundamental and powerful paradigm for optimizing an agent's policy while incorporating notions of safety.
A prevalent safe RL approach is based on a constrained criterion, which seeks to maximize the expected cumulative reward subject to specific safety constraints.
Despite recent effort to enhance safety in RL, a systematic understanding of the field remains difficult. This challenge stems from the diversity of constraint representations and little exploration of their interrelations.
To bridge this knowledge gap, we present a comprehensive review of representative constraint formulations, along with a curated selection of algorithms designed specifically for each formulation.
In addition, we elucidate the theoretical underpinnings that reveal the mathematical mutual relations among common problem formulations.
We conclude with a discussion of the current state and future directions of safe reinforcement learning research.
\end{abstract}

\section{Introduction}
\label{sec:introduction}

% RL is attractive
Reinforcement learning (RL, \citep{sutton1998reinforcement}) is a powerful paradigm for solving sequential decision-making problems through evaluation and improvement, which has made significant progress in many fields such as games \citep{mnih2015human,silver2016mastering,wurman2022outracing}, robotics \citep{levine2016end}, data center cooling~\citep{li2019transforming}, finance~\citep{hambly2023recent}, recommendation systems \citep{afsar2022reinforcement}, and healthcare~\citep{yu2021reinforcement}.
Recently, particular attention has been paid to RL for its use in the fine-tuning of large language models (LLMs) under the name of reinforcement learning from human feedback (RLHF, \cite{ouyang2022training}).
RL is a general concept that can be employed in many domains and has been spreading to various disciplines as a useful paradigm.

% Safe RL is important + a more organized summary of application ==> Revised
When applying RL to real-world problems, \textit{safety} is an essential requirement \citep{amodei2016concrete}.
Thus, \textit{safe RL} has been actively studied in recent years so that the benefits of RL are realized while minimizing negative safety issues.
Promising areas for safe RL applications include robotics~\citep{pham2018optlayer}, autonomous driving~\citep{shalev2016safe}, healthcare~\citep{jia2020safe}, and many others. 
An emerging and crucial use case of safe RL involves refining LLMs using RLHF to align with human preferences.
Specifically, it is critical to prevent harmfulness or bias (e.g., toxicity, discrimination) while maintaining the helpfulness of the generated sentences.
For example, Safe RLHF~\citep{dai2023safe} is a representative approach to balance such a tradeoff.
\textit{Safe RL} is an active area of research in artificial intelligence, with extensive theoretical and practical investigations aimed at developing RL systems that exhibit safe and reliable behavior.

%  Method rather than formulation
Safe RL is inherently a broad concept with different formulations for the different aspects of real-world safety-critical problems.
\citet{garcia2015comprehensive} provided an eminent survey of safe RL and categorized its optimization criteria into four groups: 1) constrained criteron~\citep{geibel2006reinforcement}, 2) worst-case criteron~\citep{heger1994consideration}, 3)
risk-sensitive criteron~\citep{borkar2002q}, and 4) others (e.g., \textit{r-squared, value-at-risk}).
This paper focuses on the constrained criterion of safe RL as a growing and powerful group of methods to optimize policies under safety constraints.

% Formulation of Safe RL has not been discussed
A noteworthy aspect of safe reinforcement learning research based on constrained optimization criteria is the existence of diverse representations for safety constraints, with little analysis the interrelationships and theoretical connections among these various formulations.
Safe RL research in the last decade has focused on developing new algorithms, i.e. \textit{how to solve the problem} while pursuing performance. 
New algorithms have continuously been developed under various formulations, thereby making it increasingly challenging to stay abreast of advancements in the field.
There are several recent survey papers on safe RL \citep{kim2020safe,brunke2021safe,liu2021policy,gu2022review,zhao2023state_survey}, but they also focus on methods rather than formulations.
As formulation represents the initial phase in comprehending safe RL or implementing algorithms in practical scenarios, it becomes imperative for the community to comprehensively survey the existing literature and lay the groundwork for acquiring a systematic comprehension.

\paragraph{Our contributions.}

This paper provides a comprehensive survey focusing on constraint \textit{formulations} in safe reinforcement learning and introduces representative algorithms for each formulation.
Furthermore, we discuss the relationships between various constraint formulations by defining three theoretical notions: \textit{transformability}, \textit{generalizability},  and \textit{conservative approximation}.
Specifically, we present theoretical results demonstrating that there exist two problems, termed \textbf{Identical or More General Safe RL (\iomgnospace)} problems, into which other common problems can be either transformed or conservatively approximated.
The main contribution of this paper is to bridge the gaps between the safe RL problems with appropriate algorithms by organizing existing research with a focus on constraint formulation.

\section{Preliminaries}
\label{sec:preliminary}

We consider safe RL problems modeled using a constrained Markov decision process (CMDP, \cite{altman1999constrained}), which is formally represented as
\begin{align}
    \label{eq:sec2_cmdp}
    \cmdp \coloneqq \underbrace{\tuple{\calS, \calA, \calP, H, r, \gamma_r, \rho}}_{\text{Standard MDP}\,(\calM)} \cup \ \mathcal{C},
\end{align}
where $\calS \coloneqq \set{s}$ is a state space, $\calA \coloneqq \set{a}$ is an action space, and $\calP: \calSxA \rightarrow \Delta(\calS)$ is the state transition, where $\calP(s' \,| \,s, a)$ is the probability of transition from state $s$ to state $s'$ when action $a$ is taken.
Note that $\Delta(X)$ denotes the probability simplex over the set $X$.
In addition, $H \in \mbZ_{+}$ is the (fixed) finite length of each episode, $r: \calSxA \rightarrow [0, 1]$ is the reward function, $\gamma_r \in [0, 1)$ is the discount factor for the reward, and $\rho \in \Delta(\calS)$ is the initial state distribution.
A key difference from a standard MDP $\calM$ lies in the existence of an additional tuple $\calC$ to represent safety constraints, which will be referred to as a ``constraint tuple'' in the rest of this paper.\footnote{Though this paper considers finite-horizon discounted CMDPs, key ideas can be extended to infinite and/or undiscounted cases.} 

A policy $\pi: \calS \rightarrow \Delta(\calA)$ is a function to map from states to distribution over actions.
Let $\Pi$ denote a policy class.
Given a policy $\pi \in \Pi$, the value function is defined as
\begin{align} \label{eq:value_s}
    V^\pi_{r, h}(s) \coloneqq &~\Epi{\sum_{h'=h}^H \gamma_r^{h'} r(s_{h'}, a_{h'}) \lmid s_h = s},
\end{align}
where the expectation $\mbE_{\pi}$ is taken over the random state-action sequence $\{(s_{h'}, a_{h'})\}_{h'=h}^H$ induced by the policy $\pi$ and the CMDP $\cmdp$.
Since the initial state $s_0$ is sampled from $\rho$, we slightly abuse the notation and define
\begin{align} \label{eq:value_initial}
    V^\pi_r(\rho) \coloneqq \Erho{V_{r, 0}^\pi(s)}.
\end{align}

Crucially, this paper deals with safe RL involving the ``constrained'' policy optimization problem.
A policy must be within the feasible policy space $\widehat{\Pi} \subseteq \Pi$ that satisfies the safety constraint based on the given constraint tuple $\calC$.
Therefore, the optimal policy $\pi^\star: \calS \rightarrow \Delta(\calA)$ is defined as
\begin{align}
    \label{eq:sec2_safe_rl}
    \pi^\star \coloneqq \argmax_{\pi \in \widehat{\Pi}} V^\pi_r(\rho).
\end{align}

An overall sequence for solving safe RL problems based on a constrained criterion is illustrated in Figure~\ref{fig:sequence}, which consists of 1) problem formulation and 2) policy optimization.
An interesting yet complicated point for understanding safe RL research is that there are several formulations for representing the constrained tuple $\calC$ and the feasible policy space $\widehat{\Pi}$ depending on the problem settings or applications that researchers have in their minds.
In the next section, we will review seven common safety constraint representations that have been well-studied in safe RL literature.

\tikzstyle{b} = [draw, thick, fill=white, rectangle, minimum height=2em, text width=22em, text centered]

\begin{figure}[t]
\centering
\begin{tikzpicture}[auto, node distance=2cm,>=latex']
    % P1
    \node [b] (p1) {
        \textbf{Step 1: Problem Formulation}
        \begin{equation*}
            \blue{\max_{\pi \in \Pi} V^\pi_r(\rho)} 
            \quad \text{subject to} \quad
            \red{\textbf{[Safety Constraint]}}
        \end{equation*}
        \begin{itemize}
            \item \blue{Typical RL objective}
            \item \red{Diverse \textbf{Safety Constraint} representations}
        \end{itemize}
        };
    % P2
    \node [b, below of=p1, node distance=2.8cm] (p2) {
        \textbf{Step 2: Policy Optimization \vspace{0pt}}  \\
        \begin{itemize}
            \item Either use an existing algorithm suitable for the problem setup or develop a new algorithm
        \end{itemize}
    };
    \draw [->, thick] (p1) -- node[name=2to1] {} (p2);
\end{tikzpicture}
\caption{A typical sequence for solving safe RL problems based on constrained criteria. Due to the diversity of safety constraint representations and little discussion on their interrelations, it is not easy to understand safe RL research systematically. Unlike existing survey papers that focus on \textit{methods}, we aim to provide a comprehensive survey from the perspective of \textit{formulations} on safe RL.}
\label{fig:sequence}
\end{figure}
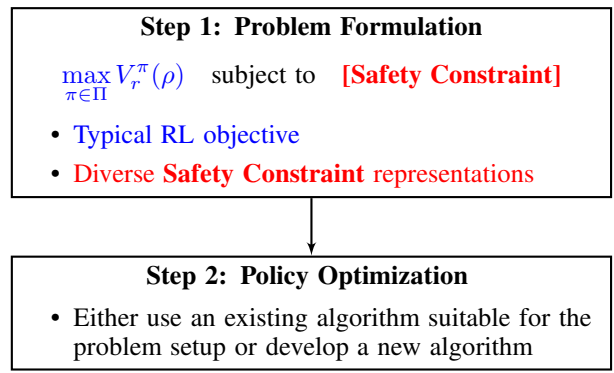
\section{Common Constraint Formulations}
\label{sec:common_formulations}

This paper deals with safe RL based on a constrained criterion.
As presented shortly, for a CMDP $\cmdp$, common formulations discussed in this paper can be written as
\begin{align}
    \label{eq:safe_rl_fc}
    \max_{\pi \in \Pi} V^\pi_r(\rho) \quadstquad f_{\calC}(\pi) \le 0.
\end{align}
In the rest of this paper, we call $f_{\calC}: \Pi \rightarrow \mbR$ a \textit{safety constraint function (SCF)} and let $\calF_{\calC} = \set{f_{\calC}}$ denote a class of SCFs.
Crucially, $\calC$ contains the parametric information of the safety constraint; hence, $\calF_{\calC}$ is regarded as a class of SCFs that can be represented with $\calC$.
Note that, though $f_\calC$ often depends on the parameters defined in the standard MDP $\calM$ (e.g., $\calP$), we omitted it from the notation for simplicity.
The feasible policy space $\widehat{\Pi}$ depends on the structure and parameters of $f_{\calC}$; hence is, we explicitly represent it as
\begin{align}
    \widehat{\Pi}(f_{\calC}) \coloneqq \set{\pi \in \Pi \mid f_{\calC}(\pi) \le 0}.
\end{align}

In the following subsections, we will introduce common safety constraint representations and associated algorithms while organizing based on the aforementioned notations.
Representative existing studies and algorithms are summarized according to the problem formulations in Table~\ref{tab:formulations}.

% Problem 1
\subsection{Expected Cumulative Safety Constraint}

One of the most popular safe RL formulations is to represent the safety constraint using the same structure of the value function in terms of reward (i.e., $V_r^\pi(\rho)$).
As with the reward function, we first define the value function for safety 
in a state $s$ at time $h$, denoted as
\begin{align*}
    V_{c, h}^\pi(s) \coloneqq &~\Epi{\sum_{h'=h}^H \gamma_c^{h'} c(s_{h'}, a_{h'}) \lmid s_h = s}.
\end{align*}
where $c: \calSxA \rightarrow [0, 1]$ is a safety cost function.\footnote{Without loss of generality, we assume that the safety cost function $c$ is bounded between $0$ and $1$ for all $(s, a)$ as with the reward function $r$. Especially, the assumption that the safety cost is \textit{positive} is important for our theoretical analyses.}
We then take expectation with respect to the initial state distribution $\rho$; that is, for a policy $\pi \in \Pi$, we have 
\begin{align*}
    V^\pi_c(\rho) \coloneqq &~\Erho{V_{c, 0}^\pi(s)}.
\end{align*} 
Therefore, we define the following safe RL problem based on the value function regarding the safety cost function.
\begin{problem}
    \label{problem:expected_cumulative_safety}
    Let a constraint tuple be $\calC \coloneqq \tuple{c, \gamma_c, \xi}$, where $c: \calSxA \rightarrow [0, 1]$ is the safety cost function, $\gamma_c \in [0, 1)$ is the discount factor for the safety, and $\xi \in \mbR_+$ is the safety threshold.
    Then,
    \begin{align}
        \max_\pi V^\pi_r(\rho) \quadstquad V_c^\pi(\rho) \le \xi.
    \end{align}
\end{problem}

\begin{remark}
    In Problem~\ref{problem:expected_cumulative_safety}, the SCF $f_{\calC}: \Pi \rightarrow \mbR$ is defined as
    \begin{align*}
        f_{\calC}(\pi) \coloneqq V_c^\pi(\rho) - \xi.
    \end{align*}
\end{remark}

\noindent
A reason why Problem~\ref{problem:expected_cumulative_safety} is popular is that the safety constraint has a high affinity with the value function in terms of reward.
For example, when $\gamma_r = \gamma_c$, for any $\lambda \in \mbR$, we have
\begin{align*}
    V_{r + \lambda c}^\pi(\rho) = V_r^\pi(\rho) + \lambda \cdot V_c^\pi(\rho),
\end{align*}
which leads to so-called Lagrangian-based methods~\citep{altman1999constrained}.
Thus, since the theoretical properties of the problem have been studied, several algorithms with theoretical guarantees on optimality or safety have been proposed (e.g., \cite{ding2021provably}, \cite{bura2022dope}).
Practically, many well-known algorithms, such as constrained policy optimization (CPO, \cite{AchiamHTA17}) or reward-constrained policy optimization (RCPO, \cite{tessler2019reward}, are built upon Problem~\ref{problem:expected_cumulative_safety}.

% Problem 2
\subsection{State Constraint}
Another popular formulation involves leveraging the ``state'' constraints so that an agent avoids visiting a set of unsafe states, which is well-suited for the case where an autonomous robot needs to behave safely in a hazardous environment (e.g., a disaster site).
This type of formulation has been widely adopted by previous studies on safe-critical robotics tasks \citep{thananjeyan2021recovery,thomas2021safe,turchetta2020safe,wang2023enforcing}, which is written as follows:

\begin{problem}
    \label{problem:state}
    Let a constraint tuple be $\calC \coloneqq \tuple{S_{\text{unsafe}}, \gamma_c, \xi}$, where $S_{\text{unsafe}} \subseteq \calS$ is the set of unsafe states, $\gamma_c \in [0, 1)$ is the discount factor for the safety cost function, and $\xi \in \mbR_+$ is the safety threshold.
    Then,
    \begin{align*}
        \max_\pi V^\pi_r(\rho) \quadstquad \Epi{\sum_{h=0}^H \gamma_c^h \mbI(s_h \in S_{\text{unsafe}})} \le \xi,
    \end{align*}
    where $\mbI(\cdot)$ is the indicator function.
\end{problem}

\begin{remark}
    In Problem~\ref{problem:state}, the SCF $f_{\calC}: \Pi \rightarrow \mbR$ is defined as
    \begin{align*}
        f_{\calC}(\pi) \coloneqq \Epi{\sum_{h=0}^H \gamma_c^h \mbI(s_h \in S_{\text{unsafe}})} - \xi.
    \end{align*}
\end{remark}

% Problem 3
\subsection{Joint Chance Constraint}
Policy optimization under joint chance constraints has been studied especially in the field of control theory such as \citet{ono2015chance} and \citet{pfrommer2022safe}.
This type of safety-constrained policy optimization problem is typically formulated as follows:
\begin{problem}
    \label{problem:chance}
    Let a constraint tuple be $\calC \coloneqq \tuple{S_{\text{unsafe}}, \xi}$, where $S_{\text{unsafe}} \subseteq \calS$ is the set of unsafe states and $\xi \in \mbR_+$ is the safety threshold.
    Then,
    \begin{align*}
        \max_\pi V^\pi_r(\rho) \quadstquad \Ppi{\bigvee_{h=0}^{H} s_h \in S_\text{unsafe}} \le \xi,
    \end{align*}
    where $\mbP_\pi$ represents a probability that is computed over the random state-action sequences $\set{(s_{h'}, a_{h'})}_{h'=h}^H$ induced by the policy $\pi$ and the CMDP $\cmdp$.
\end{problem}

\begin{remark}
    In Problem~\ref{problem:chance}, the SCF $f_{\calC}: \Pi \rightarrow \mbR$ is defined as
    \begin{align*}
        f_{\calC}(\pi) \coloneqq \Ppi{\bigvee_{h=0}^{H} s_h \in S_\text{unsafe}} - \xi.
    \end{align*}
\end{remark}

\noindent
It is quite challenging to directly solve the Problem~\ref{problem:chance} characterized by joint chance-constraints; thus, most of the previous work does not directly deal with this type of constraint and uses some approximations or assumptions.
For example, \citet{pfrommer2022safe} assume a known linear time-invariant dynamics.
Also, \citet{ono2015chance} conservatively approximate the joint chance constraint into the constraint with an additive structure as in Problem~\ref{problem:state} using the following inequality (for more details, see the proof of Theorem~\ref{theorem:2}):
\begin{align*}
    \Ppi{\bigvee_{h=1}^{H} s_h \in S_\text{unsafe}}
    \le \Epi{\sum_{h=1}^H  \mathbb{I}(s_h \in S_\text{unsafe})}.
\end{align*}

\subsection{Expected Instantaneous Safety Constraint with Time-variant Threshold}
Though Problems~\ref{problem:expected_cumulative_safety}, \ref{problem:state}, and \ref{problem:chance} formulate the long-term safety constraint with additive structures or joint chance constraints, there are a few papers (e.g., \cite{pham2018optlayer}) that focus on ``instantaneous'' safety constraint.
In this problem, an agent is required to satisfy a safety constraint at every time step.

\begin{problem}
    \label{problem:expected_gse}
    Let a constraint tuple be $\calC \coloneqq \tuple{c, \{\xi_h\}_{h=0}^H}$, where $c: \calSxA \rightarrow [0, 1]$ is the safety cost function, and $\xi_h \in \mbR_+$ is the safety threshold for time step $h \in [H]$.
    Then,
    \begin{align*}
        \max_\pi V^\pi_r(\rho) \quadstquad \Epi{c(s_h, a_h)} \le \xi_h, \ \forall h \in [H].
    \end{align*}
\end{problem}

\begin{remark}
    In Problem~\ref{problem:expected_gse}, the SCF $f_{\calC}: \Pi \rightarrow \mbR$ is defined as
    \begin{align*}
        f_{\calC}(\pi) \coloneqq \max_{h \in [H]} \Bigl(\Epi{c(s_h, a_h)} - \xi_h \Bigr).
    \end{align*}
\end{remark}

\subsection{Almost Surely Cumulative Safety Constraint}

Unlike Problem~\ref{problem:expected_cumulative_safety} where the expectation $\mbE_\pi$ regarding the safety constraint is taken, we often want to guarantee safety \textit{almost surely} (i.e., probability of $1$).
This problem has been recently studied in \cite{sootla2022enhancing} and \cite{sootla2022saute}, which is based on a stricter safety notion.
In such cases, the safe RL problem is formulated as follows:
\begin{problem}
    \label{problem:almost_surely}
    Let a constraint tuple be $\calC \coloneqq \tuple{c, \gamma_c, \xi}$, where $c: \calSxA \rightarrow [0, 1]$ is the safety cost function, $\gamma_c \in [0, 1)$ is the discount factor for the safety, and $\xi \in \mbR_+$ is the safety threshold.
    Then,
    \begin{align*}
        \max_\pi V^\pi_r(\rho) \quadstquad \Ppi{\sum_{h=0}^H \gamma_c^h c(s_h, a_h) \le \xi} = 1.
    \end{align*}
\end{problem}

\begin{remark}
    In Problem~\ref{problem:almost_surely}, the SCF $f_{\calC}: \Pi \rightarrow \mbR$ is defined as
    \begin{align*}
        f_{\calC}(\pi) \coloneqq 1 - \Ppi{\sum_{h=0}^H \gamma_c^h c(s_h, a_h) \le \xi}.
    \end{align*}
\end{remark}

\subsection{Almost Surely Instantaneous Safety Constraint with Time-invariant Threshold}
Some existing studies formulate safe RL problems via an instantaneous constraint, attempting to ensure safe exploration \citep{sui2015safe} even during the learning stage while aiming for extremely safety-critical RL applications such as planetary exploration and medical treatment~\citep{turchetta2016safe,wachi2018safe,wachi2020safe,wang2023enforcing}.
Such studies require the agent to satisfy the following instantaneous safety constraint at every time step.

\begin{problem}
    \label{problem:instantaneous_fixed}
    Let a constraint tuple be $\calC \coloneqq \tuple{c, \xi}$, where $c: \calSxA \rightarrow [0, 1]$ is the safety cost function, and $\xi \in \mbR_+$ is the threshold.
    Then,
    \begin{align*}
        \max_\pi V^\pi_r(\rho) \quadstquad \Ppi{c(s_h, a_h) \le \xi} = 1, \ \forall h \in [H].
    \end{align*}
\end{problem}

\begin{remark}
    In Problem~\ref{problem:instantaneous_fixed}, the SCF $f_{\calC}: \Pi \rightarrow \mbR$ is defined as
    \begin{align*}
        f_{\calC}(\pi) \coloneqq 1 - \prod_{h=0}^H \Ppi{c(s_h, a_h) \le \xi}.
    \end{align*}
\end{remark}
\noindent
This formulation is also related to notions in the control theory, called control barrier functions \citep{ames2016control,cheng2019end} or Lyapunov functions \citep{berkenkamp2017safe}.
Note that while these functions are typically required to be positive, it is possible to use the same formulation in Problem~\ref{problem:instantaneous_fixed} by defining them as $-c$ and setting $\xi = 0$.

\subsection{Almost Surely Instantaneous Safety Constraint with Time-variant Threshold}
As a similar formulation to Problem~\ref{problem:instantaneous_fixed}, \cite{wachi2023safe}.
has recently introduced a problem called the generalized safe exploration (GSE) problem, which is written as follows:
\begin{problem}
    \label{problem:as_gse}
    Let a constraint tuple be $\calC \coloneqq \tuple{c, \{\xi_h\}_{h=0}^H}$, where $c: \calSxA \rightarrow [0, 1]$ is the safety cost function, and $\xi_h \in \mbR_+$ is the safety threshold for time step $h \in [H]$.
    Then,
    \begin{align*}
        \max_\pi V^\pi_r(\rho) \quadstquad \Ppi{c(s_h, a_h) \le \xi_h} = 1, \forall h \in [H].
    \end{align*}
\end{problem}

\begin{remark}
    In Problem~\ref{problem:as_gse}, the SCF $f_{\calC}: \Pi \rightarrow \mbR$ is defined as
    \begin{align*}
        f_{\calC}(\pi) \coloneqq 1 - \prod_{h=0}^H \Ppi{c(s_h, a_h) \le \xi_h}.
    \end{align*}
\end{remark}

\noindent
This formulation is quite similar to Problem~\ref{problem:instantaneous_fixed} with the only difference being that the safety threshold $\xi_h$ is time-variant.
An apparent benefit of Problem~\ref{problem:as_gse} compared to Problem \ref{problem:instantaneous_fixed} is that we can cover a wider range of applications such that the speed limit changes during driving.
Additionally, it goes beyond that and offers us more important advantages regarding the theoretical relations with Problem~\ref{problem:almost_surely}, which we will present shortly in Theorem~\ref{theorem:3}.

\subsection{Other Constrained Formulations}

In the line of safe RL work based on constraints, various attempts were left unexplored and not integrated into our theoretical framework within this paper. 

A notable instance of such problem formulations defines a safety constraint using the \textit{variance} of the return~\citep{tamar2012policy} which is represented as %
\begin{equation*}
    \max_\pi V^\pi_r(\rho) \quadstquad \textrm{Var}\sqbla{V_r^\pi(\rho)} < \xi.
\end{equation*}
The \textit{variance} of the return is related to the Sharpe ratio~\citep{sharpe1966mutual} – the ratio between the expected profit and its standard deviation.
Thus, this formulation is particularly useful in financial appliactions~\citep{meng2019reinforcement}.

Also, conventional RL algorithms depend on \textit{ergodicity} assumption; that is, any state $s$ is eventually reachable from any other state $s'$ by following a suitable policy.
This assumption does not hold in many practical applications since an agent cannot recover on its own after catastrophic actions.
\cite{moldovan2012safe} removed the ergodicity assumption and proposed an algorithm in which an agent is required to guarantee returnability to the initial safe state.

While this paper has dealt with \textit{numerical} safety constraints, there have been attempts to represent them using a certain language.
For example, \cite{fulton2018safe} or \cite{hasanbeig2020cautious} represent the safety constraint via formal languages such as linear temporal logic.
This constraint representation is quite useful for leveraging human knowledge in RL, which leads to powerful solutions such as shielding methods~\citep{alshiekh2018safe,li2020robust}.
Finally, \cite{yang2021safe} uses natural language to represent safety constraints.
Natural language is one of the most powerful mediums yet friendly representations to the general public. 
Given the recent remarkable progress and utilities of LLMs, this approach would be promising for applying safe RL to real-world AI systems.
\begin{table*}[t!]
    \centering
    \begin{small}
    \begin{tabular}{ccllccc}
        \toprule
        \multirow{2}{*}[-1mm]{Problem} & \multirow{2}{*}[-1mm]{Type} & \multirow{2}{*}[-1mm]{Representative Work} & \multirow{2}{*}[-1mm]{Algorithm} & \multicolumn{2}{c}{Theoretical Guarantee} & \multirow{2}{*}[-1mm]{Open Source Software (OSS)} \\
        \cmidrule(l){5-6}
        & & & & Optimality & Safety & \\
        \midrule
        \multirow{21}{*}{Problem~\ref{problem:expected_cumulative_safety}} 
            & \multirow{18}{*}{Online} & \cite{AchiamHTA17} & CPO & $-$ & $-$ & A, SSA, FSRL, SafePO, OmniSafe \\
            & & \multirow{2}{*}{\cite{ray2019benchmarking}} & TRPO-Lagrangian & $-$ & $-$ & A, SSA, FSRL, SafePO, OmniSafe \\
            & & & PPO-Lagrangian & $-$ & $-$ & A, SSA, FSRL, SafePO, OmniSafe \\
            & & \cite{tessler2019reward} & RCPO & $-$ & $-$ & A, SafePO, OmniSafe \\
            & & \cite{liu2020ipo} & IPO & $-$ & $-$ & A, OmniSafe \\
            & & \cite{yang2020projection} & PCPO & $-$ & $-$ & A, SafePO, OmniSafe \\
            & & \cite{stooke2020responsive} & PID-Lagrangian & $-$ & $-$ & A, SafePO, OmniSafe \\
            & & \cite{zhang2020first} & FOCOPS & $-$ & $-$ & A, FSRL, SafePO, OmniSafe \\
            & & \cite{ding2020natural} & NPG-PD & Y & C & $-$ \\
            & & \cite{bharadhwaj2021conservative} & CSC & $-$ & $-$ & A \\
            & & \cite{ding2021provably} & OPDOP & Y & C & $-$ \\
            & & \cite{bai2021achieving} & CSPDA & Y & C & $-$ \\
            & & \cite{as2021constrained} & LAMBDA & $-$ & $-$ & A \\
            & & \cite{xu2021crpo} & CRPO & Y & C & OmniSafe \\
            & & \cite{yu2022towards} & SEditor & $-$ & $-$ & A \\
            & & \cite{bura2022dope} & DOPE & Y & T and C & $-$ \\
            & & \cite{liu2022constrained} & CVPO & Y & C & A, FSRL \\
            & & \cite{zhang2022penalized} & P3O & $-$ & $-$ & A, OmniSafe \\
            \cmidrule(l){2-7}
            & \multirow{5}{*}{Offline} & \cite{le2019batch} & CBPL & $-$ & T and C & A \\
            & & \cite{lee2021coptidice} & COptiDICE & $-$ & T & A, OSRL, OmniSafe \\
            & & \cite{wu2021offline} & CMOMDPs & Y & T and C & $-$ \\
            & & \cite{xu2022constraints} & CPQ & $-$ & T & A, OSRL \\
            & & \cite{liu2023constrained} & CDT & $-$ & T & A, OSRL \\
        \midrule
        \multirow{4}{*}{Problem~\ref{problem:state}}
            & \multirow{4}{*}{Online} & \cite{turchetta2020safe} & CISR & $-$ & $-$ & A \\
            & & \cite{thomas2021safe} & SMBPO & $-$ & C & A \\
            & & \cite{thananjeyan2021recovery} & Recovery RL & $-$ & $-$ & A \\
            & & \cite{wang2023enforcing} & $-$ & $-$ & T and C & A \\
        \midrule
        \multirow{4}{*}{Problem~\ref{problem:chance}}
            & \multirow{4}{*}{Online} & \cite{ono2015chance} & CCDP & $-$ & T and C & $-$ \\
            & & \cite{pfrommer2022safe} & $-$ & Y & T and C & $-$ \\
            & & \cite{mowbray2022safe} & $-$ & $-$ & T and C & A \\
            & & \cite{kordabad2022safe} & $-$ & $-$ & T and C & $-$ \\
        \midrule
        \multirow{4}{*}[-1mm]{Problem~\ref{problem:expected_gse}}
            & \multirow{3}{*}{Online} & \cite{pham2018optlayer} & OptLayer & $-$ & T and C & A \\
            & & \cite{amani2021safe} & SLUCB & Y & T and C & $-$ \\
            & & \cite{zhao2023state} & SCPO & Y & C & $-$ \\
            \cmidrule(l){2-7}
            & Offline & \cite{amani2022doubly} & Safe-DPVI & Y & T and C & $-$ \\
        \midrule
        \multirow{2}{*}{Problem~\ref{problem:almost_surely}}
            & \multirow{2}{*}{Online} & \cite{sootla2022saute} & Saut{\'e} RL & Y & C & A, SafePO, OmniSafe \\ 
            & & \cite{sootla2022enhancing} & Simmer RL & Y & C & A, SafePO, OmniSafe \\
        \midrule
        \multirow{8}{*}{Problem~\ref{problem:instantaneous_fixed}}
            & \multirow{8}{*}{Online} & \cite{turchetta2016safe} & SafeMDP & $-$ & T and C & A \\
            & & \cite{berkenkamp2017safe} & SMbRL & $-$ & T and C & A \\
            & & \cite{fisac2018general} & $-$ & $-$ & T and C & $-$ \\
            & & \cite{wachi2018safe} & SafeExpOpt-MDP & $-$ & T and C & A \\
            & & \cite{dalal2018safe} & SafeLayer & $-$ & T and C & A \\
            & & \cite{cheng2019end} & RL-CBF & $-$ & T and C & A \\
            & & \cite{wachi2020safe} & SNO-MDP & Y & T and C & A \\
            & & \cite{wang2023enforcing} & $-$ & $-$ & C & $-$ \\
        \midrule
        \multirow{2}{*}{Problem~\ref{problem:as_gse}}
            & \multirow{2}{*}{Online} & \cite{shi2023near} & LSVI-NEW & Y & T and C & $-$\\ 
            & & \cite{wachi2023safe} & MASE & Y & T and C & $-$ \\
        \bottomrule
    \end{tabular}
    \end{small}
    \caption{Common safe RL formulations based on the constrained criterion and associated representative work. Type indicates whether each safety RL is based on online or offline RL settings. In the Theoretical Guarantee column, \textbf{Y} indicates the (near-)optimality of the policy obtained by an algorithm. Also, \textbf{T} means that safety is guaranteed during training, and \textbf{C} means that safety is guaranteed after convergence. Note that offline algorithms are inherently safe during training since there is no interaction between the agent and the environment. In the OSS column, \textbf{A} means a public authors' implementation exists, and \textbf{SSA} is an abbreviation of the Safety Starter Agent repository~(\citet{ray2019benchmarking}, \url{https://github.com/openai/safety-starter-agents}). Also, \textbf{FSRL} (\citet{liu2023datasets}, \url{https://github.com/liuzuxin/FSRL}),  \textbf{OSRL} (\citet{liu2023datasets}, \url{https://github.com/liuzuxin/OSRL}), \textbf{SafePO} (\citet{omnisafe}, \url{https://github.com/PKU-Alignment/Safe-Policy-Optimization}), and \textbf{OmniSafe} (\citet{omnisafe}, \url{https://github.com/PKU-Alignment/omnisafe}) are recent and actively maintained repositories for online and offline safe RL, which will lead to the ease of the process of adopting safe RL algorithms.}
    \label{tab:formulations}
\end{table*}

\tikzstyle{block} = [draw, thick, fill=white, rectangle, minimum height=2em, text width=19em]

\newcommand{\mybullet}{\bullet\ }
\newcommand{\ptitle}[1]{\textbf{Problem #1} \vspace{1mm}}

\begin{figure*}[t]
\centering
\begin{tikzpicture}[auto, node distance=2cm,>=latex']
    % P1
    \node [block] (p1) {
        \ptitle{1}($\Epionly$ + Cumulative) \\
        $\mybullet \calC = \tuple{c, \gamma_c, \xi}$ \\
        $\mybullet f_\calC(\pi) = V_c^\pi(\rho) - \xi$
        };
    % P2
    \node [block, below of=p1, node distance=2.3cm] (p2) {
        \ptitle{2}($\Epionly$ + Cumulative) \\
        $\mybullet \calC = \tuple{S_\text{unsafe}, \gamma_c, \xi}$ \\
        $\mybullet f_\calC(\pi) = \Epi{\sum_{h=0}^H \gamma_c^h \mbI(s_h \in S_{\text{unsafe}})} - \xi$
        };
    % P3
    \node [block, below of=p2, node distance=2.6cm] (p3) {
        \ptitle{3}(Joint Chance Constraint) \\
        $\mybullet \calC = \tuple{S_\text{unsafe}, \xi}$ \\
        $\mybullet f_{\calC}(\pi) \coloneqq \Ppi{\bigvee_{h=0}^{H} s_h \in S_\text{unsafe}} - \xi$
        };
    % P7
    \node [block, right of=p2, node distance=10cm] (p4) {
        \ptitle{4}($\Epionly$ + Instantaneous) \\
        $\mybullet \calC = \tuple{c, \set{\xi_h}_{h=0}^H}$ \\
        $\mybullet f_{\calC}(\pi) \coloneqq \max_{h \in [H]} \Bigl(\Epi{c(s_h, a_h)} - \xi_h \Bigr)$
        };

    % Lines
    \draw [->, thick] (p2) -- node[name=2to1] {Transformable} (p1);
    \draw [->, thick, dotted] (p3) -- node[name=2to1] {Conservative Approximation} (p2);
    \draw [->, thick] (p1) -| node[pos=0.5] {} node [near end] {Transformable} (p4);
    \draw [->, thick] (p2) -- node[name=1to7] {Transformable} (p4);
    \draw [->, thick, dotted] (p3) -| node[pos=0.5] {} node [near end] {Conservative Approximation} (p4);
\end{tikzpicture}
\caption{Relations among common safe RL formulations based on $\mbE_\pi$ and the one with chance constraints.}
\label{fig:1237}
\end{figure*}
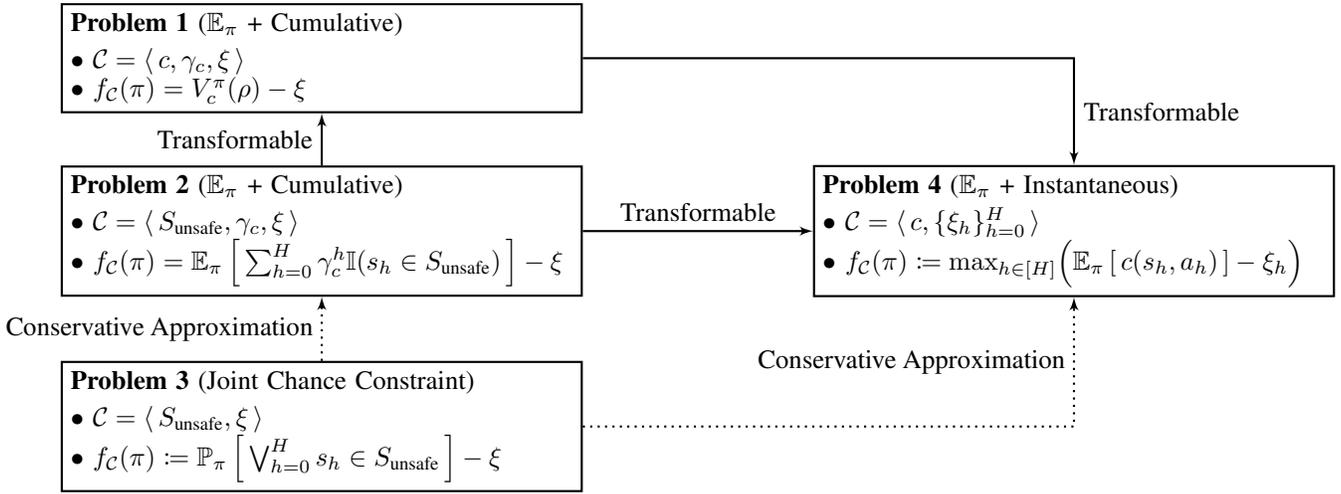

\begin{figure*}[t]
\centering
\hspace{20pt}
\begin{tikzpicture}[auto, node distance=2cm,>=latex']
    % P4
    \node [block] (p5) {
        \ptitle{5}($\Ppionly$ + Cumulative) \\ 
        $\mybullet \calC = \tuple{c, \gamma_c, \xi}$ \\
        $\mybullet f_\calC(\pi) = 1 - \Ppi{\sum_{h=0}^H \gamma_c^h c(s_h, a_h) \le \xi}$
        };
    % P5
    \node [block, below of=p5, node distance=2.2cm] (p6) {
        \ptitle{6}($\Ppionly$ + Instantaneous) \\
        $\mybullet \calC = \tuple{c, \xi}$ \\
        $\mybullet f_\calC(\pi) = 1 - \prod_{h=0}^H \Ppi{c(s_h, a_h) \le \xi}$
        };
    % P6
    \node [block, right of=p5, node distance=10cm] (p7) {
        \ptitle{7}($\Ppionly$ + Instantaneous) \\
        $\mybullet \calC = \tuple{c, \{\xi_h\}_{h=0}^H}$ \\
        $\mybullet f_{\calC}(\pi) \coloneqq 1 - \prod_{h=0}^H \Ppi{c(s_h, a_h) \le \xi_h}$
        };
    % Lines
    \draw [->, thick] (p6) -| node[pos=0.5] {} node [near end] {Transformable} (p7);
    \draw [->, thick] (p5) -- node[] {Transformable} (p7);
\end{tikzpicture}
\caption{Relations among common safe RL formulations based on $\mbP_\pi$ (i.e., almost-surely constraints).}
\label{fig:456}
\end{figure*}
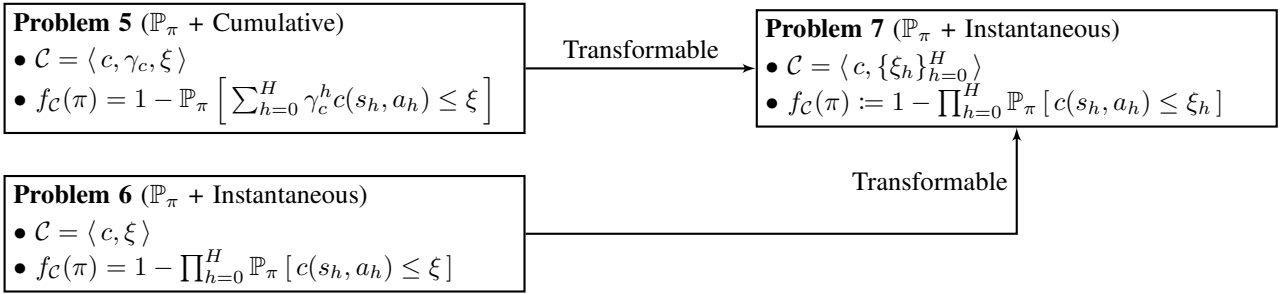
\section{Theoretical Relations Among Common Constraint Formulations of Safe RL}
\label{sec:relationship}

In this section, we provide theoretical understandings regarding the interrelations among common safe RL formulations presented in Section~\ref{sec:common_formulations}.

\subsection{Definitions}

We first introduce and define two important notions called \textit{transformability} and \textit{generalizability}.
Also, we define a notion called \textit{conservative approximation}.

\begin{definition}[Transformability]
    Let $\cmdp_1$ and $\cmdp_2$ denote two CMDPs that are respectively characterized by different constraint tuples $\calC_1$ and $\calC_2$.
    Let $\calF_{\calC_1}$ and $\calF_{\calC_2}$ denote two different classes of SCFs based on $\calC_1$ and $\calC_2$, respectively.
    For any SCF $f_{\calC_1} \in \calF_{\calC_1}$, if there exists $f_{\calC_2} \in \calF_{\calC_2}$ such that
    \begin{align*}
        \widehat{\Pi}(f_{\calC_1}) = \widehat{\Pi}(f_{\calC_2}),
    \end{align*}
    then we say that the problem characterized by $\cmdp_1$ can be transformed into that characterized by $\cmdp_2$.
\end{definition}

\begin{definition}[Generalizability]
    Let $N \in \mbZ_{+}$ denote a positive integer.
    Let $\set{\cmdp_i}_{i=1}^N$ denote a set of $N$ CMDPs that are respectively characterized by different constraint tuples $\calC_1, \calC_2, \ldots, \calC_N$.
    Suppose that, for all $i \in [N]$, the problem characterized by $\cmdp_i$ can be transformed into that by $\calM \cup \widetilde{\calC}$.
    Then, we call the problem characterized by $\calM \cup \widetilde{\calC}$ is an identical or more general safe RL (\iomgnospace) problem over the set of problems respectively characterized by $\cmdp_1, \cmdp_2, \ldots, \cmdp_N$.
\end{definition}

\begin{definition}[Conservative Approximation]
    Let $\cmdp_1$ and $\cmdp_2$ be two CMDPs that are respectively characterized by different constraint tuples $\calC_1$ and $\calC_2$.
    Let $\calF_{\calC_1}$ and $\calF_{\calC_2}$ respectively denote two classes of SCFs based on $\calC_1$ and $\calC_2$.
    For any SCF $f_{\calC_1} \in \calF_{\calC_1}$, if there exists $f_{\calC_2} \in \calF_{\calC_2}$ such that
    \begin{align*}
        \widehat{\Pi}(f_{\calC_1}) \supset \widehat{\Pi}(f_{\calC_2}),
    \end{align*}
    then we say that the problem characterized by $\cmdp_2$ is a conservative approximation of that characterized by $\cmdp_1$.
\end{definition}

\subsection{Preliminary Lemmas}

To present the main theoretical results, we first list three necessary lemmas as preliminaries.
The first lemma is based on \cite{sootla2022saute} and \cite{wachi2023safe}, which describes a theoretical connection between additive and instantaneous safety constraints.

%% Lemma 1
\begin{lemma}
    \label{lemma:saute}
    Define a new variable $\eta_h$ meaning the remaining safety budget associated with the discount factor $\gamma_c$ such that
    \begin{equation}
    \label{eq:eta_1}
        \eta_h \coloneqq \gamma_c^{-h} \cdot \left(\xi - \sum_{h'=0}^{h-1} \gamma_c^{h'} c(s_{h'}, a_{h'}) \right), \ \forall h \in [H].
    \end{equation}
    Then, the following relation between additive and instantaneous constraints holds:
    \begin{alignat*}{2}
        \sum_{h'=0}^{h-1} \gamma_c^{h'} c(s_{h'}, a_{h'}) \le \xi 
        \iff c(s_h, a_h) \le \eta_h, \quad \forall h \in [H].
    \end{alignat*}
\end{lemma}

\begin{proof}
    By definition, the new variable $\eta_h$ satisfies the following recurrence formula:
    \begin{alignat}{2}
    \label{eq:eta_2}
        \eta_{h+1} = \gamma_c^{-1} \cdot \left(\, \eta_h - c(s_h, a_h) \, \right) \quad \text{with} \quad \eta_0 = \xi.
    \end{alignat}
  Combining \eqref{eq:eta_2} and $c(s, a) \ge 0, \forall (s, a) \in \calS \times \calA$, we have
  \begin{equation*}
    \eta_{H+1} \ge 0  \iff \eta_h - c(s_h, a_h) \ge 0, \quad \forall h \in [H].
  \end{equation*}
  By definition of $\eta_h$, we have
    \begin{equation}
        \sum_{h'=0}^{H} \gamma_c^{h'} c(s_{h'}, a_{h'}) \le \xi \iff \eta_{H+1} \ge 0.
    \end{equation}
    Therefore, we obtain the desired lemma.
\end{proof}

%% Lemma 2
\begin{lemma}
    \label{lemma:p1to4}
    Problem~\ref{problem:expected_cumulative_safety} can be transformed into Problem~\ref{problem:expected_gse}.
\end{lemma}

\begin{proof}
    By applying Lemma~\ref{lemma:saute}, the safety constraint of Problem~\ref{problem:expected_cumulative_safety} satisfies
    \begin{equation*}
        V_c^\pi(\rho) \le \xi \iff   \Epi{c(s_h, a_h)} \le \Epi{\eta_h}, \ h \in [H]. 
    \end{equation*}
    Therefore, Problem~\ref{problem:expected_cumulative_safety} is identical to a special case of Problem~\ref{problem:expected_gse} with $\xi_h \coloneqq \mbE_\pi[\eta_h]$ for all $h \in [H]$.
    Hence, we obtain the desired lemma.
\end{proof}

%% Lemma 3
\begin{lemma}
    \label{lemma:p5to7}
    Problem~\ref{problem:almost_surely} can be transformed into Problem~\ref{problem:as_gse}.
\end{lemma}

\begin{proof}
    By applying Lemma~\ref{lemma:saute}, we have
    \begin{alignat*}{2}
        &\Ppi{\sum_{h=0}^H \gamma_c^h c(s_h, a_h) \le \xi} = 1 \\
        & \iff
        \Ppi{c(s_h, a_h) \le \eta_h} = 1, \quad \forall h \in [H].
    \end{alignat*}
    Therefore, Problem~\ref{problem:almost_surely} is identical to a special case of Problem~\ref{problem:as_gse} with $\xi_h \coloneqq \eta_h$ for all $h \in [H]$.
    Therefore, we obtain the desired lemma.
\end{proof}

\subsection{The Two \iomg Problems}

We now provide three theorems on the interrelations among the common safe RL problems.
Crucially, we show that Problems~\ref{problem:expected_gse} and \ref{problem:as_gse} can be regarded as two \iomg problems of other problems.
We also show that Problem~\ref{problem:expected_gse} with $\gamma_c = 1$ is a conservative approximation of Problem~\ref{problem:chance}. 
Conceptual illustrations are given in Figures~\ref{fig:1237} and \ref{fig:456}.

%% Theorem 1
\begin{theorem}
    Problem~\ref{problem:expected_gse} is an \iomg problem over Problems~\ref{problem:expected_cumulative_safety} and \ref{problem:state}.  
\end{theorem}
\begin{proof}
    By Lemma~\ref{lemma:p1to4}, Problem~\ref{problem:expected_cumulative_safety} can be transformed into Problem~\ref{problem:expected_gse}.
    Also, Problem~\ref{problem:state} can be easily transformed into Problem~\ref{problem:expected_cumulative_safety} by defining $c(s, a) \coloneqq \mbI(s \in S_\text{unsafe})$ for all $(s, a) \in \calS \times \calA$.
    In summary, Problem~\ref{problem:expected_gse} is an \iomg problem over Problems~\ref{problem:expected_cumulative_safety} and \ref{problem:state}.
\end{proof}

%% Theorem 2
\begin{theorem}
    \label{theorem:2}
    Problem~\ref{problem:expected_gse} with $\gamma_c = 1$ is a conservative approximation of Problem~\ref{problem:chance}.   
\end{theorem}
\begin{proof}
    This lemma mostly follows from Theorem 1 in \citet{ono2015chance}.
    Regarding the constraint in Problem~\ref{problem:chance}, we have the following chain of equations:
    \begin{align*}
        \Ppi{\bigvee_{h=1}^{H} s_h \in S_\text{unsafe}}
        &\ \le \sum_{h=1}^H \Ppi{s_h \in S_\text{unsafe}}  \\
        &\ = \sum_{h=1}^H \Epi{\mathbb{I}(s_h \in S_\text{unsafe})} \\
        &\ = \Epi{\sum_{h=1}^H  \mathbb{I}(s_h \in S_\text{unsafe})}.
    \end{align*}
    In the first step, we used Boole's inequality (i.e., $\Pr[A \cup B] \le \Pr[A] + \Pr[B]$).
    The final term is the left-hand side of the constraint in Problem~\ref{problem:state} with $\gamma_c = 1$, which implies that Problem~\ref{problem:state} is a conservative approximation of Problem~\ref{problem:chance}.
    Therefore, we obtained the desired theorem.
\end{proof}

%% Theorem 3
\begin{theorem}
\label{theorem:3}
    Problem~\ref{problem:as_gse} is an \iomg problem over Problems~\ref{problem:almost_surely} and \ref{problem:instantaneous_fixed}. 
\end{theorem}
\begin{proof}
    By Lemma~\ref{lemma:p5to7}, Problem~\ref{problem:almost_surely} can be transformed into Problem~\ref{problem:as_gse}.
    Also, Problem~\ref{problem:instantaneous_fixed} is a special case of Problem~\ref{problem:as_gse} where $\xi_h$ is a constant for all $h \in [H]$.
    Hence, Problem~\ref{problem:as_gse} is an \iomg problem over Problems~\ref{problem:almost_surely} and \ref{problem:instantaneous_fixed}.
\end{proof}
\section{Discussion}
\label{sec:discussion}

We conclude by discussing the current state and future directions of safe RL based on constrained criteria.

\subsection{Formulation and Algorithm Selection}

As mentioned earlier, when we adopt the safe RL paradigm, there are two main steps: 1) \textit{problem formulation} and 2) \textit{policy optimization}.
Based on the survey results, let us discuss how to solve safe RL problems via constraints.

\paragraph{Problem formulation.}
As presented in Section~\ref{sec:common_formulations}, constraint formulations in safe RL are divided into two classes: one based on $\mbE_\pi$ and the other based on $\mbP_\pi$.
Constraint formulations with expectation $\mbE_\pi$ (i.e., Problems~\ref{problem:expected_cumulative_safety} - \ref{problem:expected_gse}) represent the safety constraints using the expected value, which focuses on the ``averaged'' performance of safety.
On the other hand, constraint formulations based on $\mbP_\pi$ (i.e., Problems~\ref{problem:almost_surely} - \ref{problem:as_gse}) require an agent to guarantee safety almost surely.
When the problem is formulated based on $\mbP_\pi$, the achieved safety level is usually higher by nature.
Still, there is also a drawback that the reward performance is usually lower due to stricter safety constraints.
Which formulation to adopt should depend on which level of safety is faced with the problem.

\paragraph{Policy optimization.}
When optimizing a policy, we must select a proper algorithm that corresponds to the problem formulation.
Given the current situation, the easiest way is to use algorithms implemented in well-used OSS such as FSRL/OSRL~\citep{liu2023datasets} and SafePO/OmniSafe~\citep{omnisafe}.
Note that, as shown in Table~\ref{tab:formulations}, the algorithms implemented in the above OSS are mostly based on Problems~\ref{problem:expected_cumulative_safety} and \ref{problem:almost_surely}.
In the long run, however, the algorithms based on Problems~\ref{problem:expected_gse} and \ref{problem:as_gse} (i.e., \iomg problems) are also promising since instantaneous constraints are easier to handle than cumulative ones both theoretically and empirically.

\paragraph{When should your agent be safe?}
Another perspective in the algorithm selection should include \textit{when} an agent needs to satisfy the safety constraint.
Existing algorithms can be divided into two classes (see ``safety'' column in Table~\ref{tab:formulations}.
The first class (w/o \textbf{T}) tries to achieve the required level of safety \textit{after convergence} while encouraging safety during training.
The algorithms in this class are usually based on Problems~\ref{problem:expected_cumulative_safety} or \ref{problem:almost_surely} where the safety constraints are represented using the additive safety cost structure.
The second class (w/ \textbf{T}) tries to guarantee safety even \textit{during training} as well as after convergence.
The algorithms in this class are typically built upon Problems~\ref{problem:instantaneous_fixed} or \ref{problem:as_gse} where the safety constraints are instantaneous.
Which class of algorithm you should choose depends solely on the problem and application to be addressed.

\subsection{Online RL vs. Offline RL}

As shown in Table~\ref{tab:formulations}, most of the existing safe RL literature considers ``online'' RL settings where an agent interacts with the environment while learning its policy.
A major advantage of safe ``online'' RL is that policies can be trained from scratch without data collected previously.
On the other hand, ``offline'' RL \citep{levine2020offline} is a framework to train a policy from a fixed amount of pre-collected data, potentially solving the fundamental issues regarding safety or risk.
Offline RL is well-suited in the context of safe RL because the agent does not interact with the real environment during training; thus, the policy training does not essentially pose any risk.
Hence, \textit{safe offline RL} is a promising approach to achieve safety in RL if enough data or high-fidelity simulation is available.
Although most of the existing safe offline RL literature is based on Problem~\ref{problem:expected_cumulative_safety} (e.g., \cite{le2019batch}, \cite{lee2021coptidice}) as shown in Table~\ref{tab:formulations}, such research direction is expected to expand to other problem settings.

\section{Conclusion}
\label{sec:conclusion}

This paper provides a comprehensive survey of safe reinforcement learning based on constrained optimization criteria, with a particular focus on problem \textit{formulations}.
We present seven common constraint formulations and their associated representative algorithms in Section~\ref{sec:common_formulations}.
Additionally, a curated selection of relevant literature is summarized according to the problem formulation in Table~\ref{tab:formulations}.
In Section~\ref{sec:relationship}, we examine the theoretical relations among these formulations, offering readers a deeper understanding of safe RL.
Furthermore, we describe the current status of safe RL research and outline potential future directions.
Through this survey paper, we aim to foster a systematic understanding of constraint formulations and encourage further fundamental and applied research in safe reinforcement learning.

\bibliographystyle{named}
\bibliography{ref}

\end{document}